\documentclass[twoside]{article}
\usepackage{graphicx} 
\usepackage{amsmath, amssymb}
\usepackage{amsthm}
\usepackage{mathtools}
\usepackage{multirow}

\usepackage{caption} 

\theoremstyle{plain}
\newtheorem{theorem}{Theorem}[section]
\newtheorem*{theorem*}{Theorem}

\theoremstyle{definition}

\theoremstyle{remark}

\usepackage[utf8]{inputenc} 
\usepackage[T1]{fontenc}    
\usepackage{hyperref}       
\usepackage{url}            
\usepackage{booktabs}       
\usepackage{amsfonts}       
\usepackage{nicefrac}       
\usepackage{microtype}      
\usepackage{xcolor}         

\usepackage[capitalize,noabbrev]{cleveref}

\usepackage[preprint]{aistats2026}

\usepackage[round]{natbib}
\setcitestyle{authoryear, open={(}, close={)}}

\begin{document}

\runningtitle{Tighter PAC-Bayes Risk Certificates}

\twocolumn[

\aistatstitle{Some theoretical improvements on the tightness of PAC-Bayes risk certificates for neural networks}

\aistatsauthor{ Diego García-Pérez \And Emilio Parrado-Hernández \And  John Shawe-Taylor }

\aistatsaddress{Dept. of Signal Theory \\ and Communications\\ UC3M \\ \texttt{diegoarg@pa.uc3m.es}
\And Dept. of Signal Theory \\ and Communications\\UC3M \\ \texttt{eparrado@ing.uc3m.es} 
\And AI Centre \\ Dept. of Computer Science \\ UCL \\ \texttt{j.shawe-taylor@ucl.ac.uk}  }

]


\begin{abstract}

    This paper presents four theoretical contributions that improve the usability of risk certificates for neural networks based on PAC-Bayes bounds. First, two bounds on the KL divergence between Bernoulli distributions enable the derivation of the tightest explicit bounds on the true risk of classifiers across different ranges of empirical risk. The paper next focuses on the formalization of an efficient methodology based on implicit differentiation that enables the introduction of the optimization of PAC-Bayesian risk certificates inside the loss/objective function used to fit the network/model. The last contribution is a method to optimize bounds on non-differentiable objectives such as the 0-1 loss. These theoretical contributions are complemented with an empirical evaluation on the MNIST and CIFAR-10 datasets. In fact, this paper presents the first non-vacuous generalization bounds on CIFAR-10 for neural networks.
\end{abstract}

\section{Introduction}
\label{intro}

A critical open problem in machine learning is certifying models' performance, particularly evaluating and guaranteeing their generalization capabilities. This is vital in regulated domains such as healthcare (e.g., diagnostic tools), finance (e.g., credit risk models), and autonomous systems (e.g., self-driving vehicles). The true risk of a classifier, defined as the probability of incorrectly classifying any test observation, would be an ideal certification. However, such risk turns out to be impossible to compute as in general the true underlying data distribution is unknown. The PAC-Bayesian framework offers a robust alternative to derive such certificates by producing non-trivial upper bounds on the true risk without requiring costly cross-validations or splitting data into training and testing subsets \citep{alquier2023userfriendlyintroductionpacbayesbounds}.  

The PAC-Bayes framework operates by defining two distributions over the hypothesis space: the prior, which must be independent from the data used to produce the certificate, and the posterior, which may depend on said data. The key to estimating the generalization capability lies in combining the empirical risk of the posterior with the quantification of the divergence between these two distributions. While this paper studies the generalization of the posterior, there are also ``de-randomization'' techniques that study the generalization of a sample from the posterior~\citep{catoni2007pac, blanchard2007occam, viallard2024general}.

PAC-Bayesian analyses have traditionally been effective with simpler models such as kernel methods~\citep{parrado2012pac}, where the prior and posterior distributions can be characterized with relatively few parameters. Extending these techniques to certify deep neural networks (NNs) presents significant challenges due to the complexity and size of modern architectures~\citep{dziugaite2017computing}. The state of the art PAC-Bayes bounds on the true risk of NNs are significantly weaker than those estimations obtained through traditional train/test split methods. 

An approach that is not discussed in this paper, even though it produces models with strong performance and tight bounds, is the use of data-dependent priors~\citep{parrado2012pac, dziugaite2021role}. The reason for this is that when addressing more complex tasks such as classification on CIFAR-10, this approach becomes equivalent to a train-test split pipeline with the added complexity of the PAC-Bayes framework. This can be seen in the work of~\citep{perez2021tighter}, where such approach yields posteriors with a negligible divergence from their priors and no increased performance.

Like most of the recent work in PAC-Bayes, this paper builds on Maurer's bound~\citep{maurer2004bound}. Maurer's bound assumes a bounded loss function, and consequently most results in this line of research are constrained to bounded losses. Some exceptions are~\citep{haddouche2021pac, zhang2023improving}.

In this context this paper aims at augmenting the usability of PAC-Bayes certificates for NNs with the following contributions:

\begin{itemize}
  \item New relaxations of Maurer's bound~\citep{maurer2004bound} that, to our knowledge, yield the tightest closed-form PAC-Bayesian risk bounds.
  \item A new training objective for SGD derived directly from Maurer's bound.
  \item A theoretical analysis of the KL-attenuation trick that motivates a new method to optimize risk bounds on non-differentiable loss functions.
\end{itemize}


\section{Obtaining risk certificates} \label{section: obtaining risk certs}

Consider a set of $n$ samples $S=\{(X_i,Y_i)\}_{i=1}^n$ i.i.d. from an unknown distribution $Z := X \times Y$, where $X_i\in \mathbb R^d$ and $Y\in \{0,1\}$. The goal of our learning algorithm is to find a mapping $h:X\rightarrow Y$ from a hypotheses space $\mathcal{H}$ such that $h(X_t)$ is a good approximation to $Y_t$ for any $(X_t,Y_t) \in Z$. Let us denote by $Q_0$ a data-independent prior distribution over $\mathcal{H}$, and by $Q$ the posterior data-dependent distribution over $\mathcal{H}$ from which the learning algorithm selects the final $h$. The precision of each prediction made by $h$ is measured with a bounded loss function $l: \mathcal{H} \times Z \rightarrow [0, 1]$, and the performance of $h$ with the risk $L(h)$, defined as its expected loss over $Z$: 
\[
L(h) := \mathbb{E}_{z \sim Z}[l(h, z)]
\]
The definition of risk can be extended to distributions over hypotheses: $L(Q) := \mathbb{E}_{h \sim Q}[L(h)] $.
Analogously, the empirical risk of $h$ w.r.t. $S$ is the average of the loss over the elements of $S$: $\hat{L}_S(h) = \frac{1}{n} \sum_{i=1}^n{l(h,S_i)}$. Likewise, $\hat{L}_S(Q) := \mathbb{E}_{h \sim Q}[\hat{L}_S(h)]$. 

The Kullback-Leibler divergence between two distributions is represented by $\mbox{KL}(\cdot || \cdot)$. If these distributions are Bernoulli with means $p$ and $q$, respectively, we will use:
\[
    \mbox{kl}(p||q) = p\log\left (\frac{p}{q}\right) + (1-p)\log\left(\frac{1-p}{1-q}\right)
    \quad p, q \in (0,1)
\]

The classical ``PAC-Bayes-kl'' bound with confidence $1-\delta$ was introduced by~\cite{langford2001bounds}, and slightly improved by~\cite{maurer2004bound}:

\begin{theorem}[PAC-Bayes-kl] \label{thm: PAC-Bayes-kl}
    Let $S=\{(X_i,Y_i)\}_{i=1}^n \overset{\text{i.i.d.}}{\sim} Z$ with $n \geq 8$. For any distribution $Q_0$ independent from $S$, any distribution $Q$ and a bounded loss function $0 \leq l \leq 1$, it holds with probability $1-\delta$ that:

\begin{equation} \label{eq: kl-bound}
    \textup{kl}(\hat{L}_S(Q)||L(Q)) \leq \frac{\textup{KL}(Q||Q_0) + \log(\frac{2\sqrt{n}}{\delta})}{n}
\end{equation}

\end{theorem}

For the remainder of the paper, $K$ denotes the right-hand side of ineq.~\eqref{eq: kl-bound}:

\[
    K := \frac{\mbox{KL}(Q||Q_0) + \log(\frac{2\sqrt{n}}{\delta})}{n}
\]

Since $\textup{kl}(\hat{L}_S(Q)||\hat{L}_S(Q)) = 0$, we assume $0 < \hat{L}_S(Q) \leq L(Q) < 1$ when discussing any upper bound on $L(Q)$ without loss of generality.

Although $L(Q)$ cannot be computed without access to $Z$, Theorem \ref{thm: PAC-Bayes-kl} gives an upper bound on $L(Q)$. However, because the bound is not explicit on $L(Q)$, it may be bothersome in practice. A common workaround to obtain an explicit bound on $L(Q)$ is to find a lower bound to the left side of \eqref{eq: kl-bound} and solve for $L(Q)$. Another workaround is to consider the following inequality equivalent to \eqref{eq: kl-bound}
\begin{equation} \label{eq: second ineq}
    L(Q) \leq \mbox{kl}^{-1}(\hat{L}_S(Q), K)
\end{equation}
where $\mbox{kl}^{-1}$ is precisely the function that makes both inequalities equivalent, that is, $\mbox{kl}^{-1}(p, k) = \mbox{sup}\{q \in [p, 1): \mbox{kl}(p||q) \leq k\}$. In fact, due to the convexity and non-negativity of the divergences, $\mbox{kl}(p||q)=0 \iff p=q$ implies $\mbox{kl}(p||q)$ is strictly increasing for $q>p$. This means $\mbox{kl}^{-1}(p, k) := q \mbox{ such that } \mbox{kl}(p||q) \leq k \mbox{ and } q\geq p$ is equivalent and well-defined. The inversion of the binary KL of \cite{seeger2002klinv}, $\mbox{kl}^{-1}(p, k) := \Delta \mbox{ such that } \mbox{kl}(p||p+\Delta) \leq k \mbox{ and } \Delta \geq 0$, while not consistent with the inverse of a function, simplifies the proofs in this paper. The second workaround involves upper-bounding the right-hand side of \eqref{eq: second ineq} \citep{alquier2023userfriendlyintroductionpacbayesbounds,hellström2024generalizationboundsperspectivesinformation}. Moreover, this explicit bound can then be used to optimize Theorem~\ref{thm: PAC-Bayes-kl} with respect to $Q$. This is possible because, although an analytical expression for $\mbox{kl}^{-1}(p, k)$  is not known, it can be approximated numerically.


\subsection{Relaxing the inequalities} \label{subsection: inequalities}

This subsection reviews the most common relaxations of inequalities \eqref{eq: kl-bound} and \eqref{eq: second ineq}. The review is non-exhaustive, and does not cover parametric relaxations such as the PAC-Bayes-$\lambda$ bound  \citep{thiemann2017strongly}, which is closely related to KL-modulation (see subsection \ref{subsect: kl mod method}). A key point here is that any lower bound for the left-hand side of ineq.~\eqref{eq: kl-bound} that admits a closed-form solution for $L(Q)$ can be turned into an upper bound for the right-hand side of ineq.~\eqref{eq: second ineq}, and any upper bound for the right-hand side of \eqref{eq: second ineq} that admits a closed-form solution for $K$ can be turned into a lower bound for the right-hand side of \eqref{eq: kl-bound}. Since we need a closed-form solution for $L(Q)$ but not for $K$, the second approach is stronger than the first and should be the tool to derive even tighter bounds in the future. Each bound presented in this section admits a closed form for $K$, and both forms are provided.

\subsubsection{Pinsker's inequality}
 
Consider Pinsker's inequality, a classical lower bound on the KL divergence that applied to two Bernoulli distributions with means $p$ and $q$ yields $\mbox{kl}(p||q) \geq 2(q-p)^2$. Applying this lower bound to ~\eqref{eq: kl-bound} produces a version of the classic McAllester bound~\citep{mcallester1999bound} with improved constants, highlighting the potential of Maurer's bound \citep{dziugaite2017computing}:

\begin{equation} \label{eq: classic_pinsker}
    L(Q) \leq \hat{L}_S(Q) + \sqrt{\frac{K}{2}}
\end{equation}

This bound is equivalent to upper bounding the right-hand side of \eqref{eq: second ineq} by $\mbox{kl}^{-1}(p,k) \leq p + \sqrt{k/2}$.

\subsubsection{Refined Pinsker's inequality} \label{subsect: RP bound}

Applying a refinement of the Pinsker inequality, $\mbox{kl}(p||q) \geq (q - p)^2/(2q)$, to the left-hand side of \eqref{eq: kl-bound} and solving for $L(Q)$ yields the PAC-Bayes-quadratic (PBQ) bound \citep{rivasplata2019pac}:

\begin{equation} \label{eq: refined_pinsker}
    L(Q) \leq \left(\sqrt{\hat{L}_S(Q) + \frac{K}{2}} + \sqrt{\frac{K}{2}}\right)^2
\end{equation}

This bound is tighter than \eqref{eq: classic_pinsker} when and only when $L(Q) < 1/4$, which is often the case in practice. In fact, the objective functions derived from \eqref{eq: refined_pinsker} lead to much better results than those derived from \eqref{eq: classic_pinsker}~\citep{perez2021tighter}.

Notice \eqref{eq: refined_pinsker} is equivalent to upper bounding the right-hand side of \eqref{eq: second ineq} by $\mbox{kl}^{-1}(p,k) \leq p + k + \sqrt{2pk + k^2} = (\sqrt{p + k/2} + \sqrt{k/2})^2$.

\subsubsection{Tolstikhin and Seldin's (TS) bound} \label{subsect: TS bound}

\cite{tolstikhin2013bound} produce a PAC-Bayesian bound by applying $\mbox{kl}^{-1}(p,k) \leq p + 2k + \sqrt{2pk}$ to the right-hand side of \eqref{eq: second ineq}:

\begin{equation}\label{eq: ts-bound}
    L(Q) \leq \hat{L}_S(Q) + 2K + \sqrt{2\hat{L}_S(Q)K}
\end{equation}

This is equivalent to lower-bounding the left-hand side of \eqref{eq: kl-bound} by $\mbox{kl}(p||q) \geq (2q - p - \sqrt{4qp - 3p^2})/4$.

It has been described as ``better''~\citep{alquier2023userfriendlyintroductionpacbayesbounds} and ``more refined''~\citep{hellström2024generalizationboundsperspectivesinformation} than \eqref{eq: classic_pinsker}, the bound derived from the classic Pinsker inequality. However, this tighter behavior strongly depends on the values of $\hat{L}_S(Q)$ and $K$.

\begin{figure}
\centering
\includegraphics[height=6cm]{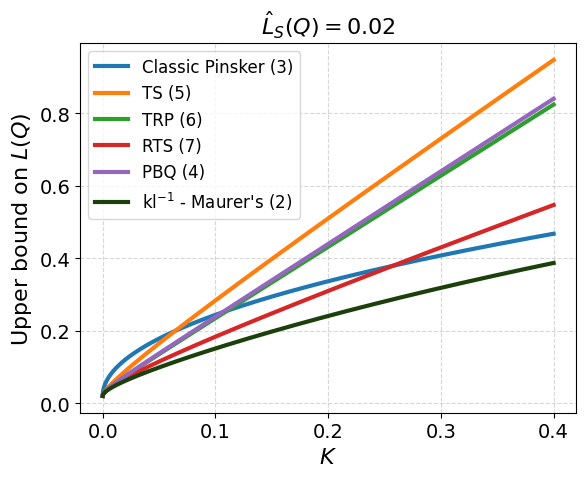}
\includegraphics[height=6cm]{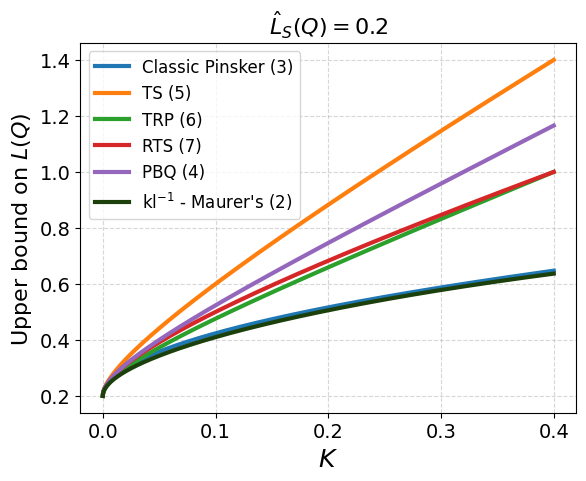}
\caption{Different upper bounds on the true risk of the posterior over the hypothesis, $L(Q)$, as a function of the right hand side of \eqref{eq: kl-bound}, $K$, for a small (top) and a large (bottom) empirical risk $\hat{L}_S(Q)$. The dark green curve represents Maurer's bound, of which the rest are relaxations of, so the closer to the dark green curve the better the relaxation.}
\label{fig:upperbounds}
\end{figure}


\section{Tighter PAC-Bayes bounds}  \label{section: contributions}




The first contribution of the paper is a new, universally tighter bound than the one in \eqref{eq: refined_pinsker}.

\begin{theorem} \label{thm: TRP ineq}
Let $f(p,q) := \frac{(q-p)^2}{2q(1-p)}$. Then
\begin{equation*}
  \mbox{kl}(p||q) \geq f(p,q) \qquad 0 < p \leq q < 1
\end{equation*}
\end{theorem}

\begin{proof}
See Appendix A.
\end{proof}

 The new $1-p$ term makes the improvement over \eqref{eq: refined_pinsker} uniform in $\{(p,q): 0 < p \leq q < 1\}$. Now, applying the new bound to \eqref{eq: kl-bound} and solving for $L(Q)$ yields the Tighter Refined Pinsker (TRP) bound:

\small
\begin{equation}\label{eq: TRP1}
    L(Q) \leq \left(\sqrt{\hat{L}_S(Q) + \frac{K(1-\hat{L}_S(Q))}{2}} + \sqrt{\frac{K(1-\hat{L}_S(Q))}{2}}\right)^2
\end{equation}
\normalsize

This is equivalent to upper bounding the right-hand side of \eqref{eq: second ineq} by $\mbox{kl}^{-1}(p,k) \leq p + (1-p)k + \sqrt{2p(1-p)k + k^2(1-p)^2} = (\sqrt{p + k(1-p)/2} + \sqrt{(1-p)k/2})^2$.

Figure \ref{fig:upperbounds} shows that the TRP bound is universally tighter than PBQ, and that this difference is most noticeable for larger values of $\hat{L}_S(Q)$.

\begin{figure}
\centering
\includegraphics[height=6cm]{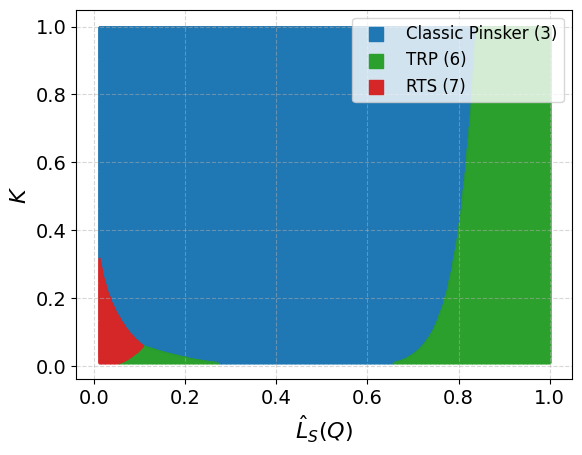}
\caption{Map with the tightest bound (out of all the bounds studied in the paper) for different values of $K$ (right-hand side of \eqref{eq: kl-bound}) and $\hat{L}_S(Q)$.}
\label{fig:tightest_bounds}
\end{figure}


The second contribution of the paper is a new, universally tighter bound than~\eqref{eq: refined_pinsker} and~\eqref{eq: ts-bound}, and tighter than~\eqref{eq: TRP1} for smaller ranges of the empirical risk $\hat{L}_S(Q)$:
\begin{equation}\label{eq: TRP2}
 L(Q) \leq \hat{L}_S(Q) + K + \sqrt{2\hat{L}_S(Q)K}
 \end{equation}
 
 This bound, which we call Refined Tolstikhin-Seldin (RTS) bound, can be obtained by lower-bounding the left-hand side of~\eqref{eq: kl-bound} by $\mbox{kl}(p||q) \geq q - \sqrt{2qp - p^2}$.
 
\begin{theorem} \label{thm: rts}
Let $f(p,q) := q - \sqrt{2qp - p^2}$. Then,
\begin{equation*}
\mbox{kl}(p||q) \geq f(p,q) \qquad 0 < p \leq q < 1
\end{equation*}
\end{theorem}

\begin{proof}
See Appendix A.
\end{proof}

Figure \ref{fig:upperbounds} shows that for a larger value of the empirical risk, the classic bound \eqref{eq: classic_pinsker} is significantly tighter than the other bounds analyzed in the paper, and closely follows the inverse of the KL. When the empirical risk is smaller, RTS is the tightest for small values of $K$ and the classic bound \eqref{eq: classic_pinsker} is the tightest for increasing values of $K$. With respect to the bounds proposed in the paper, Figure \ref{fig:upperbounds} shows that they are universally tighter than their counterparts: RTS is tighter than TS bound and TRP is tighter than PBQ.

Moreover, Figure \ref{fig:tightest_bounds} shows that just three of the bounds analyzed in the paper suffice to cover the whole range of values of $K$ and $\hat{L}_S(Q)$: The classic Pinsker relaxation of \eqref{eq: classic_pinsker} covers the largest area while the bounds introduced in this paper are the tightest for small values of $K$ and $\hat{L}_S(Q)$, as well as for large values of $\hat{L}_S(Q)$.


\section{Objective functions that optimize risk certificates} \label{sec: optim}

Closed-form approximations to the inverse of the KL enable the minimization of the bound on $L(Q)$ through gradient descent. This way, one can fit NNs with specific objective functions that optimize any of the risk certificates expressed by these bounds \citep{perez2021tighter}. However, relaxing~\eqref{eq: kl-bound} to obtain said closed-form expression can result detrimental. An alternative can be the direct calculation of the gradient of the bound on $L(Q)$ with respect to $\theta$, the parameters that define $Q$.

Let $q$ be the bound on $L(Q)$ given in \eqref{eq: kl-bound}, for an empirical loss $p := \hat{L}_S(Q)$ and a right-hand term $K$:

\begin{equation} \label{eq: maurer objective}
    \nabla_\theta q = \xi \left (   \left ( \log \frac{1-p}{1-q}-\log \frac{p}{q} \right ) \nabla_\theta p + \nabla_\theta K \right )
\end{equation}

where all the gradients are computed \textit{w.r.t.} $\theta$, and 

\begin{equation*} 
    \xi = \left (
           \frac{1-p}{1-q} - \frac{p}{q}
         \right )^{-1}
\end{equation*}

See Appendix B for more details. 

This result is a generalization of Equation (5) of~\citep{germain2009pac}. Germain et al. apply a similar method to optimize a bound on the probit loss of a spherical gaussian posterior with identity covariance matrix, but the technique seems to have gone unnoticed by more recent works.

\subsection{Insights on the optimization of non-differentiable functions} \label{subsect.: insights}

The reason why a better approximation $f(\hat L_S(Q), K) \approx \mbox{kl}^{-1}(\hat L_S(Q), K)$ produces a better risk certificate through gradient descent is not that the error $|f - \mbox{kl}^{-1}|$ is lower, but that the direction of the gradient of $L(Q)$ with respect to $\theta$ deviates less from the one derived directly from Maurer's bound. This gradient can be regarded as a weighted sum of the gradients of $\hat{L}_S(Q)$ and $K$ (of course the coefficients of said weighted sum would depend heavily on the values of $\hat{L}_S(Q)$ and $K$). Then, the role of our approximation $f$ will be to estimate the ratio between both coefficients of the weighted average. In those cases where $f$ consistently over or underestimates this ratio, corrections should be applied.

This idea of correcting the relative weights of empirical loss and KL is not new. It has been used many times to derive risk certificates in the Bayes-by-backpropagation literature~\citep{blundell2015weight, perez2021tighter}. However, the motivations given in these works to justify this correction are fundamentally empirical. The strength of the correction is commonly optimized through grid search, or just fixed to a value that ``seems to work well''. We argue that an unbiased objective function does not need an arbitrary ``KL-attenuating coefficient'' (inequalities in Sections \ref{section: obtaining risk certs} and \ref{section: contributions} are biased by nature, as they are bounds), and if one seeks to optimize a function with unknown gradient through a surrogate function with known gradient, then all efforts should be devoted to the estimation of the relationship between both gradients, rather than to the optimization of the scaling of the KL term in the gradient of the surrogate function.

Problem \eqref{eq: maurer objective} expresses the gradient of a training objective as a function of $q$, $p$, $\nabla p$ and $\nabla K$. The terms $q$, $p$ and $\nabla K$ are computable, but we require the empirical risk $p$ to be differentiable for $\nabla p$ to exist. This is also a requisite in the traditional machine learning setting. A common procedure to optimize a non-differentiable loss, like the zero-one loss $l^{01}$, is to assume that the minima of the non-differentiable loss will align with those of a differentiable one, like the cross-entropy loss $l^{xe}$, and optimize for the latter expecting to optimize for the former. However, this is not the case with PAC-Bayes bounds.

To work around this issue, consider the following setting: let $l^t$ be a non-differentiable loss function whose bound on the true risk $L^t(Q)$ we seek to optimize, and $l^{d}$ be a differentiable loss function. Let us denote by $\hat{L}_S^t$ and $\hat{L}_S^d$ the empirical risks corresponding to $l^t$ and $l^d$, respectively. To use $l^{d}$ as a surrogate loss, one must build a differentiable function $r(\hat{L}_S^{d}) \approx \hat{L}_S^t$. Now, applying the chain rule to \eqref{eq: maurer objective} yields the following gradient $\nabla q_{t}$ for loss $l^t$:

\begin{equation} \label{eq: maurer acc objective}
    \nabla q_{t} \approx \xi_t \left (   \left ( \log \frac{1-p_t}{1-q_t} - \log \frac{p_t}{q_t} \right ) r^\prime(p_d)\nabla p_d + \nabla K \right )
\end{equation}

where subscripts $d$ and $t$ assigns the corresponding term to the implicit function for loss $l^d$ or $l^t$, respectively, and 

\[
   \xi_t = \left (
           \frac{1-p_t}{1-q_t} - \frac{p_t}{q_t}
         \right )^{-1}
\]

\subsection{The KL-modulating method} \label{subsect: kl mod method}

 Consider you have found an arbitrarily good estimator $r$ in equation~\eqref{eq: maurer acc objective}. For a fixed $\theta \in \Theta$, the gradient of the arbitrarily good resulting objective function can be computed as $c_L \nabla \hat{L}_S^{d} + c_K \nabla K$. Now, consider the gradient of a different objective function evaluated at the same $\theta$, $b_L \nabla \hat{L}_S^{d} + b_K \nabla K$. In general, for any $\theta \in \Theta$, there exists a KL-attenuating coefficient $\eta$ that verifies $c_L \nabla \hat{L}_S^{d} + c_K \nabla K \propto b_L \nabla \hat{L}_S^{d} + \eta b_K \nabla K$, which is $\eta = c_Kb_L/(c_Lb_K)$. Consequently, for any $\theta^*$ that is a local minimum of the first objective function, there is a value of $\eta$ that makes $\theta^*$ a critical point of the second objective function. In the case of good minima, it is more than reasonable to assume that $\theta^*$ will still be a minimum in the worse KL-attenuated objective function. Therefore, applying the KL modulating method with the optimal $\eta$ to a surrogate function is sufficient to find any good local minimum, including the best possible solution.

This method is commonly known as ``KL-attenuating'', because in the literature the optimal value of $\eta$ is generally less than 1. However, the experiments in Section~\ref{section: experiments} show that the optimal value of $\eta$ to minimize a risk certificate may be greater than one, so we rename it ``KL modulating''.

\begin{figure}[!h]
\centering
\includegraphics[height=6cm]{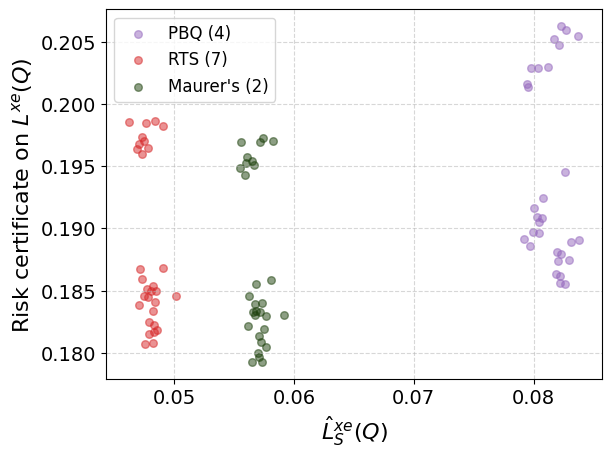}
\includegraphics[height=6cm]{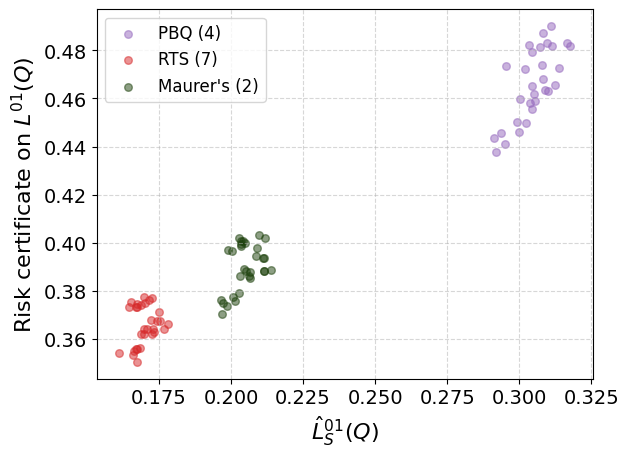}
\caption{Risk certificates on cross-entropy loss (top) and zero-one loss (bottom) on different experiment runs. Color indicates training objective. The training objective functions work with $L^{xe}$ in both plots, which explains why the RTS bound outperforms the stronger bypassed Maurer's bound in the right plot.}
\label{fig:experiment_rand_xe}
\end{figure}

\section{Experiments} \label{section: experiments}
This section gives some insights on the quality of the theoretical contributions of the paper by showing their performance in the certification of NNs trained on MNIST \citep[][CC3.0 License]{deng2012mnist} and CIFAR-10 \citep[][MIT License]{krizhevsky2009learning}. \footnote{Code to replicate all experiments in this paper is available at \texttt{github.com/Diegogpcm/pacbayesgradients}}
 
In all experiments, the prior distribution on the classifiers, $Q_0$, is the product of the individual priors on each weight of the network. Each individual prior is a Gaussian $\mathcal G(m, \sigma^2)$ where $\sigma$ is explored in $[0.02, 0.08]$ (every prior has the same $\sigma$ in a given experiment run) and $m$ is randomly sampled from $\mathcal G(0, 1/n_{in})$, where $n_{in}$ is the dimension of the input to the corresponding layer. 

The posterior $Q$ is a Gaussian with diagonal covariance. The mean and covariance of $Q$ are initialized to be those of $Q_0$ and optimized through 100 epochs of stochastic gradient descent on the corresponding objective function. We evaluate three of these objective functions: 
\begin{itemize}
    \item $f_{pbq}$, the PBQ bound of \eqref{eq: refined_pinsker}, that serves as baseline for the proposals of this work.
    \item $f_{rts}$, the RTS bound of \eqref{eq: TRP2}.
    \item $f_{mb}$, resulting from the application of the procedure described in \eqref{eq: maurer objective} to the Maurer's bound of \eqref{eq: kl-bound}.
\end{itemize}
All risk certificates in the results are obtained with a confidence factor of $\delta=0.025$. Moreover, the empirical risks in the bounds are estimated using Monte Carlo with $150.000$ models sampled from $Q$, following the same procedure as \cite{perez2021tighter}. 

Some experiments provide confidence intervals calculated with 10 different seeds for the random number generator. Reported values are formatted as $\mu \pm 2\sigma$, where $\mu$ represents the mean of the results and $\sigma$ the standard deviation. Notice these confidence intervals have nothing to do with the confidence in the bound itself, which is already set to be $1-\delta = 0.975$. The relative size of these intervals also gives insight on the sensitivity of the final value of the bound to randomized elements of the learning algorithm, such as the prior initialization and the ordering of the samples during training. 

All experiment runs are done on NVIDIA RTX 4090 GPUs and each run takes a few minutes.

\subsection{MNIST}

The first set of experiments consists of fitting 3-layer MLPs with 600 neurons in each hidden layer on the MNIST dataset. 

\begin{table}[h!]
\centering
\caption{Results for the three objective functions used to train the MLP with the MNIST data. The first five rows show information about the tightness of the risk certificate of the posterior $Q$, namely the bound for the cross-entropy loss ($L^{xe}$), the bound for the zero-one loss ($L^{01}$), plus the corresponding empirical risks ($\hat{L}_S^{01}(Q)$ and $\hat{L}_S^{xe}(Q)$) and penalty ($\textup{KL}/n$) used for their computation. Next rows show the accuracy, measured as the $l^{xe}$ or the $l^{01}$ averaged over the test set, of 4 classification strategies implemented with each model (Gibbs classifier, deterministic classifier, ensemble and average of the prior).}

\begin{tabular}{c c c c c} 
\hline
 - & Metric & $f_{pbq}$ & $f_{rts}$ & $f_{mb}$\\ [0.5ex] 
 \hline
 \multirow{3}{*}{Risk bound} & $L^{xe}$ & 0.171 & 0.167 & \textbf{0.165} \\
                               & $L^{01}$ & 0.425 & \textbf{0.335} & 0.356 \\
                               & $\mbox{KL}/n$ & 0.033 & 0.065 & 0.052 \\
 \hline

 \multirow{1}{*}{Emp. R} & $\hat{L}_S^{xe}(Q)$ & 0.082 & \textbf{0.050} & 0.058 \\
 \multirow{1}{*}{bound} & $\hat{L}_S^{01}(Q)$ & 0.301 & \textbf{0.174} & 0.207 \\
 \hline
 \multirow{2}{*}{Gibbs} & $xe$ loss & 0.081 & \textbf{0.048} & 0.057 \\
                               & 01 loss & 0.292 & \textbf{0.163} & 0.197 \\
 \hline
 \multirow{2}{*}{$\mathbb E(Q)$} & $xe$ loss & 0.045 & \textbf{0.025} & 0.030 \\
                               & 01 loss & 0.141 & \textbf{0.096} & 0.110 \\
 \hline
 \multirow{2}{*}{Ensemble} & $xe$ loss & 0.0019 & \textbf{0.0011} & 0.0013 \\
                               & 01 loss & 0.133 & \textbf{0.093} & 0.103 \\
 \hline
 \multirow{1}{*}{$\mathbb E(Q_0)$} & 01 loss & 0.879 & 0.879 & 0.879 \\
 \hline
\end{tabular}
\label{table:metrics1}
\end{table}

Table \ref{table:metrics1} shows that $f_{rts}$ achieves the tightest certificate for the zero-one loss, and also serves as core model for the most accurate classifiers.

Figure \ref{fig:experiment_rand_xe} shows that the proposed RTS bound and the bypass procedure of \eqref{eq: maurer objective} achieve risk certificates significantly tighter than the state-of-the-art PBQ bound. This is particularly marked when the risk certificate involves the zero-one loss but the objective functions work with the cross-entropy loss (bottom plot). With respect to the performance of the contributions of this paper, Figure~\ref{fig:experiment_rand_xe} and Table~\ref{table:metrics1} show that the a priori stronger theoretical result corresponding to $f_{mb}$ achieves marginally better results than $f_{rts}$ in the bound of $L^{xe}$ while losing against it in every other metric. This is not a surprise; Section \ref{subsect.: insights} shows that if the gradient of the bound on $L(Q)$ is expressed as $c_L \nabla \hat{L}_S^{xe} + c_K \nabla K$, $f_{pbq}$ overestimates the ratio $c_K / c_L$ in Maurer's bound, while $f_{rts}$ underestimates it, effectively applying the ``KL-attenuating trick''. The consequence of this can be clearly seen in the bottom plot of Figure \ref{fig:experiment_rand_xe}.

\begin{table}[h!]
\centering
\caption{Results for the three modified objective functions used to train the MLP with the MNIST data. The objective functions are modified to directly optimize $l^{01}$ through the application of the procedure in \eqref{eq: maurer acc objective}. The objective function based on Maurer's bound obtains a better bound on $L^{01}$, which is the target objective in this experiment.} 
\small
\begin{tabular}{c c c c c} 
\hline
 - & Metric & $\Tilde{f}_{mb}$ & $\Tilde{f}_{rts}$ & $\Tilde{f}_{pbq}$ \\ [1pt] 
 \hline
 \multirow{6}{*}{Risk bound} & \multirow{2}{*}{on $L^{xe}$} & 0.205  & 0.221  & \textbf{0.182}  \\
                                                   & & $\pm$ 0.003 & $\pm$ 0.003 & $\pm$ 0.006  \\
                              &  \multirow{2}{*}{on $L^{01}$} & \textbf{0.325}  & 0.329  & 0.345  \\
                                                  &  & $\pm$ 0.006 & $\pm$ 0.005 & $\pm$ 0.021  \\
                              &  \multirow{2}{*}{$\mbox{KL}/n$} & 0.136  & 0.164  & 0.083  \\
                                                  &  & $\pm$ 0.003 & $\pm$ 0.003 & $\pm$ 0.001  \\
\hline
                               & $l^{01}(Q_0)$ & \multicolumn{3}{c}{0.899 $\pm$ 0.042}\\
  \hline                                                       
\end{tabular}
\normalsize
\label{table:metrics_acc}
\end{table}

\subsubsection{Optimizing for accuracy}\label{subsection: optim4acc}

The objective functions used to obtain the results in Table~\ref{table:metrics1} had a bounded version of the cross-entropy loss as optimization target. As discussed in Section~\ref{subsect.: insights}, to optimize the bound for the zero-one loss we must estimate $\hat{L}_S^{01}$ as a function of $\hat{L}_S^{xe}$. Figure \ref{fig:xeloss_vs_acc} makes a strong case to choose $r$ to be linear, and use the approximation $m\hat{L}_S^{xe} \approx \hat{L}_S^{01}$ in \eqref{eq: maurer acc objective}. A reasonable value for $m$ is estimated by dividing a rolling average of $L_{01}$ by a rolling average of $L_{xe}$. We denote by $\Tilde{f}_{pbq}$, $\Tilde{f}_{rts}$ and $\Tilde{f}_{mb}$ the objective functions resulting by applying this procedure to $f_{pbq}$, $f_{rts}$ and $f_{mb}$, respectively. As expected, Table \ref{table:metrics_acc} shows improved risk certificates for the 01 loss, and worsened risk certificates for the bounded cross-entropy loss when the objective functions pursue the optimization of the zero-one loss.

\begin{figure}[h!]
\centering
\includegraphics[height=6cm]{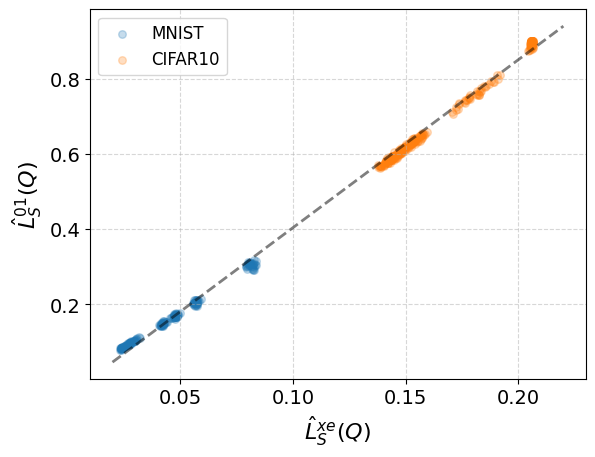}
\caption{Observed values of $\hat{L}_S^{01}$ plotted against the value of $\hat{L}_S^{xe}$ for the same experiment run. An approximately linear relationship is strongly implied.}
\label{fig:xeloss_vs_acc}
\end{figure}

\begin{figure}[h!]
\centering
\includegraphics[height=6cm]{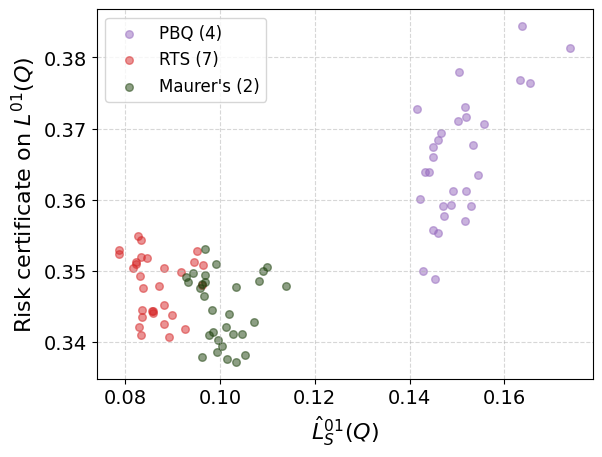}
\caption{Risk certificates computed for each MLP (one MLP per objective function and explored length scale of the prior $Q_0$) using the zero-one loss. The training objective functions work with $L^{01}$.}
\label{fig:experiment_rand_acc}
\end{figure}

Figure \ref{fig:experiment_rand_acc} shows that when applying the correction method to optimize for accuracy, the objective corresponding to the proposed stronger theoretical result, $\Tilde{f}_{mb}$, obtains slightly tighter risk certificates than the other proposal, $\Tilde{f}_{rts}$, and significantly tighter than the baseline $\Tilde{f}_{pbq}$.

\begin{table}[h!]
\centering
\caption{The same results as in Table \ref{table:metrics1} after introducing a KL-modulating coefficient. As predicted, similar risk certificates on $L^{01}$ are attained with every method.}
\begin{tabular}{c c c c c} 
\hline
 - & Metric & $\Tilde{f}_{pbq}$ & $\Tilde{f}_{rts}$ & $\Tilde{f}_{mb}$ \\ [1pt] 
 \hline
 \multirow{3}{*}{Risk bound} & $L^{xe}$ & 0.186 & 0.191 & 0.190 \\
                               & $L^{01}$ & 0.315 & 0.315 & 0.315 \\
                               & $\mbox{KL}/n$ & 0.107 & 0.117 & 0.114 \\
  \hline                                                       

\end{tabular}
\label{table:metrics_kl_mod}
\end{table}

\begin{table*}[htp]
\centering
\caption{Results for the three modified objective functions used to train CNNs of varying sizes on CIFAR-10. The objective functions are again modified to directly optimize $l^{01}$ through the application of the procedure in \eqref{eq: maurer acc objective}. Larger models give vacuous bounds and a KL of $0$. The 01 loss of the prior, $l^{01}(Q_0)$, is also provided.}
\begin{tabular}{c c c c c c} 
\hline
 Model layers & \# params & Metric & $\Tilde{f}_{pbq}$ & $\Tilde{f}_{rts}$ & $\Tilde{f}_{mb}$ \\ [1pt] 
 \hline
 \multirow{4}{*}{4} &  \multirow{4}{*}{9946} & Bound on $L^{01}$ & 0.841 $\pm$  0.101 & 0.745 $\pm$  0.011 & \textbf{0.738} $\pm$  0.007 \\
                               && Bound on $\hat{L}_S^{01}$ & 0.802 $\pm$  0.152 & 0.626 $\pm$  0.016 & \textbf{0.587} $\pm$  0.012 \\
                               && $\mbox{KL}/n$ & 0.006 $\pm$  0.009 & 0.032 $\pm$  0.002 & 0.051 $\pm$  0.004 \\
                               && $l^{01}(Q_0)$ & \multicolumn{3}{c}{0.906 $\pm$ 0.031}\\
\hline
\multirow{4}{*}{4} &  \multirow{4}{*}{315722} & Bound on $L^{01}$ & 0.823 $\pm$  0.023 & 0.773 $\pm$  0.007 & \textbf{0.758} $\pm$  0.005 \\
                               && Bound on $\hat{L}_S^{01}$ & 0.764 $\pm$  0.041 & 0.637 $\pm$ 0.012 & \textbf{0.576} $\pm$  0.011 \\
                               && $\mbox{KL}/n$ & 0.012 $\pm$  0.006 & 0.045 $\pm$  0.005 & 0.077 $\pm$  0.007 \\
                               && $l^{01}(Q_0)$ & \multicolumn{3}{c}{0.895 $\pm$ 0.018}\\
\hline
\multirow{4}{*}{5} &  \multirow{4}{*}{12266} & Bound on $L^{01}$ & >0.9 & >0.9 & \textbf{0.756} $\pm$  0.006 \\
                               && Bound on $\hat{L}_S^{01}$ & >0.9 & >0.9 & \textbf{0.601} $\pm$  0.010 \\
                               && $\mbox{KL}/n$ & 0.000 & 0.000 & 0.056 $\pm$  0.005 \\
                               && $l^{01}(Q_0)$ & \multicolumn{3}{c}{0.900 $\pm$ 0.008}\\
\hline

\end{tabular}
\label{table:metrics_cifar}
\end{table*}

\subsubsection{Applying the KL-modulating method}

This set of experiments illustrates the capabilities of the KL-modulating method introduced in Section~\ref{subsect: kl mod method} to align the gradient of any objective function with that of the zero-one loss. Table~\ref{table:metrics_kl_mod} shows that the three objective functions end up achieving very similar values for both risk certificates, especially for the one built on the zero-one loss. 

\subsection{CIFAR-10}

This set of experiments reproduces those in Section \ref{subsection: optim4acc} on the CIFAR10 dataset. Results in Table~\ref{table:metrics_cifar} show that the objective functions introduced in this paper combined with our method for optimizing the risk bound on accuracy produce nonvacuous bounds on performance on the CIFAR-10 dataset, which is to our knowledge a first in PAC-Bayes and NNs. Deeper NNs (9+ layers) such as the ones used on CIFAR-10 in~\cite{perez2021tighter} produce vacuous bounds without data-dependent priors. In our experiments, we use shallower models consisting of 4 and 5 layers.

Additionally, in the experiments in this section as well as Section~\ref{subsection: optim4acc}, the variance of the bound on the risk is generally significantly smaller than the variance of the loss of the prior. This highlights the stability of the learning algorithm, which is a desirable property, but also indicates there is room for improvement, as it shows that good minima are distributed somewhat evenly throughout the parametric space and the KL penalty restricts the search to a very small region.

\section{Conclusions}
This paper has introduced two new explicit PAC-Bayes bounds on the true risk of binary classifiers and has demonstrated that they are tighter than the state of the art. These bounds have been used to design objective functions that enable the training of NNs with the optimization of the risk certificate as main goal. Additionally, the paper has introduced a general form of the procedure used in~\citep{germain2009pac} to directly optimize non-explicit bounds on differentiable loss functions without resourcing to a relaxation of the inverse of the KL to an equation that enables a close-form computation of the gradient. The final contribution is a procedure to optimize the bound for non-differentiable loss functions such as the 01 loss, which is tied to the KL-attenuating method used in the literature. These theoretical results have been successfully applied to the calculation of risk certificates for neural networks trained on MNIST and CIFAR-10. The latter is particularly clarifying with respect to the efficacy of the results, as the contributions in sections \ref{sec: optim} and \ref{subsect.: insights} are both necessary to derive the nonvacuous risk bounds.

The results on CIFAR-10 highlight, in our opinion, the main issue of the PAC-Bayes framework. The KL term serves as a penalty that scales with the extrinsic dimensionality of the parametric space. Neural networks have a larger difference between intrinsic and extrinsic dimensionality~\citep{ansuini2019intrinsic}, and this difference becomes larger with model depth and architecture complexity. For this reason, PAC-Bayes bounds produce comparatively weak results on NNs. With the current state of the framework, it seems very unlikely that PAC-Bayes may at some point shed light on why deep learning is so effective, or explain related phenoms such as double descent. However, future research may change this. It focuses on two directions: designing architectures that inherently align with PAC-Bayesian assumptions, reducing the gap between intrinsic and extrinsic dimensionalities, and developing architecture-aware risk certification methods tailored to modern deep learning frameworks.

\newpage


\bibliography{bibliography}

\newpage

\clearpage
\appendix
\thispagestyle{empty}

\onecolumn
\aistatstitle{Appendix}

\section{Proofs} \label{appendix: proofs}

\begin{theorem*}[\ref{thm: TRP ineq}]
Let $f(p,q) := \frac{(q-p)^2}{2q(1-p)}$.
\begin{equation*}
  \mbox{kl}(p||q) \geq f(p,q) \qquad 0 < p \leq q < 1
\end{equation*}
\end{theorem*}

\begin{proof}
Let $g_p(\Delta) = \mbox{kl}(p||p+\Delta)-f(p,p+\Delta)$, for $\Delta \in [0, 1-p)$. $g_p(0)=0$, since $\mbox{kl}(p||p)=f(p,p)=0$. We first show $g'_p(\Delta) \geq 0$:
\begin{equation*}
    \begin{aligned}
        g'_p(\Delta) &= \frac{1-p}{1-p-\Delta} - \frac{p}{p + \Delta} - \frac{\Delta}{\left(1 - p\right)\left(p + \Delta\right)} 
        + \frac{\Delta^{2}}{2 \left(1 - p\right) \left(p + \Delta\right)^{2}} \\
        &= \frac{\Delta}{\left(1 - p - \Delta\right)\left(p + \Delta\right)} - \frac{\Delta}{\left(1 - p\right)\left(p + \Delta\right)} 
        + \frac{\Delta^{2}}{2 \left(1 - p\right) \left(p + \Delta\right)^{2}} \\
        &\geq \frac{\Delta}{\left(1 - p\right)\left(p + \Delta\right)} - \frac{\Delta}{\left(1 - p\right)\left(p + \Delta\right)} 
        + \frac{\Delta^{2}}{2 \left(1 - p\right) \left(p + \Delta\right)^{2}}  \\
        &= \frac{\Delta^{2}}{2 \left(1 - p\right) \left(p + \Delta\right)^{2}}  
        \geq 0
 \end{aligned}
\end{equation*}   

Since $g$ is continuous in its domain, the Mean Value Theorem states that $g_p(\Delta) = \Delta g'_p(\xi) + g_p(0) = \Delta g'_p(\xi)$ for some $\xi \in [0, \Delta]$. However, $g'_p(\xi) \geq 0$ for all $\xi \in [0, \Delta]$, so necessarily $g_p(\Delta) \geq 0$.

Finally, it is clear that $g_p(\Delta) \geq 0$ implies $\mbox{kl}(p||q) \geq f(p,q)$ if $q=p+\Delta$, and that each point in the set $\{(p,q): 0 < p \leq q < 1\}$ is contained in the domain of $g$.

\end{proof}

\begin{theorem*}[\ref{thm: rts}]
Let $f(p,q) := q - \sqrt{2qp - p^2}$.
\begin{equation*}
\mbox{kl}(p||q) \geq f(p,q) \qquad 0 < p \leq q < 1
\end{equation*}
\end{theorem*}

\begin{proof}
Let $g_p(\Delta) = \mbox{kl}(p||p+\Delta)-f(p,p+\Delta)$, for $\Delta \in [0, 1-p)$. $g_p(0)=0$, since $\mbox{kl}(p||p)=f(p,p)=0$. We first show $g'_p(\Delta) \geq 0$:
\begin{equation*}
    \begin{aligned}
        g'_p(\Delta) &= \frac{1-p}{1-p-\Delta} - \frac{p}{p + \Delta} + \frac{p}{\sqrt{p^2 + 2p\Delta}} - 1
        = 1 + \frac{\Delta}{1-p-\Delta} - \frac{p}{p + \Delta} + \frac{p}{\sqrt{p^2 + 2p\Delta}} - 1 \\
        &\geq -\frac{p}{p + \Delta} + \frac{p}{\sqrt{p^2 + 2p\Delta}}
        = -\frac{p}{\sqrt{(p + \Delta)^2}} + \frac{p}{\sqrt{(p + \Delta)^2 - \Delta^2}}
        \geq 0
    \end{aligned}
\end{equation*}

And the rest of the proof follows exactly the end of the proof of Theorem \ref{thm: TRP ineq}.

\end{proof}

\section{Derivation of the direct optimization of the bound with implicit differentiation} \label{appendix: mb reasoning}

Surrogate training objectives are necessary to optimize a function without known gradient, but this is not the case with Maurer's bound:

\begin{equation} \label{eq: kl-bound-explicit}
    L(Q) \leq \mbox{kl}^{-1}\left(\hat{L}_S(Q), \frac{\mbox{KL}(Q||Q^0) + \log(\frac{2\sqrt{n}}{\delta})}{n}\right)
\end{equation}

Let $q$ be the bound on $L(Q)$ given in \eqref{eq: kl-bound}, for an empirical loss $p := \hat{L}_S(Q)$ and a right-hand term $K$. Bound $q$ is the solution of the following equation:

\begin{equation*}
    \mbox{kl}(p||q) = K \qquad 0 < p \leq q < 1, K \geq 0
\end{equation*}

Let us define a function $f:\mathbb{R}^3 \longrightarrow \mathbb{R}$ that analyzes the space of solutions:

\begin{equation*}
    f(p,K,q) := \mbox{kl}(p||q) - K
\end{equation*}

Since $f$ is continuously differentiable and invertible in its domain $\{(p,K,q): 0<p \leq q<1, K \geq 0\}$, the Implicit Function Theorem states that there exists a continuously differentiable function $g$ such that $f(p,K,g(p,K)) = 0$. The gradient of this implicit function can be computed with:

\begin{equation*}
    \nabla g = \begin{bmatrix}
        \frac{\partial g}{\partial p}, \frac{\partial g}{\partial K}
    \end{bmatrix}
    = 
    -\left(\frac{\partial f}{\partial q}\right)^{-1}
    \begin{bmatrix}
        \frac{\partial f}{\partial p}, \frac{\partial f}{\partial K}
    \end{bmatrix}
\end{equation*}

Using the gradient descent algorithm to minimize $q$ in the parametric space of $Q$ means gradually updating $Q$ according to the gradient of $q$ with respect to $Q$. Let $\Theta$ be the parametric space of the problem we are trying to solve by gradient descent. In other words, each particular posterior $Q$ is represented by a particular $\theta$ in $\Theta$. Therefore the optimization of the bound with respect to $Q$ demands the computation of the gradients with respect to $\theta$ by means of the chain rule:

\begin{equation} \label{eq: maurer objective all}
    \begin{aligned}
    \nabla_\theta g(p(\theta), K(\theta)) &= \nabla_\theta p(\theta) \frac{\partial g} {\partial p} + \nabla_\theta K(\theta) \frac{\partial g} {\partial K} \\
    = - \left (  \nabla_\theta p \frac{\partial f } {\partial p} + \nabla_\theta K \frac{\partial f} {\partial K}  \right )  \left ( \frac{\partial f } {\partial q} \right )^{-1}
    &= - \xi \left (  \nabla_\theta p \left ( \log \frac{p}{q} - \log \frac{1-p}{1-q} \right ) + \nabla_\theta K \right )
    \end{aligned}
\end{equation}

where all the gradients are computed w.r.t $\theta$ and

\[
 \xi=\left (
           \frac{1-p}{1-q} - \frac{p}{q}
         \right )^{-1} 
\]

\end{document}